\documentclass[11pt,times]{article}
\usepackage{amsmath,amsbsy,amsfonts,amssymb,amsthm,dsfont,fullpage,units}
\usepackage[numbers]{natbib}
\usepackage{graphicx}
\usepackage{algorithm,algorithmic,mathtools}
\usepackage{color,cases}
\usepackage{tikz}
\usetikzlibrary{calc,shapes}
\usepackage{subfigure}
\usepackage{float}

\usepackage{hyperref}
\hypersetup{
colorlinks = true,
citecolor = {blue},
citebordercolor = {white}
}

\newtheorem{theorem}{Theorem}

\title{Do GANs actually learn the distribution? An empirical study}
\author{Sanjeev Arora \and Yi Zhang}
\begin{document}
\maketitle
\begin{abstract}
Do GANS (Generative Adversarial Nets) actually learn the target distribution? The foundational paper of (Goodfellow et al 2014) suggested they do,  if they were given \textquotedblleft sufficiently large\textquotedblright\ deep nets, sample size, and computation time. A recent theoretical analysis in Arora et al (to appear at ICML 2017) raised doubts whether the same holds when discriminator has finite size. It showed that the training objective can approach its optimum value even if the generated distribution has very low support  ---in other words, the training objective is unable to prevent {\em mode collapse}.

The current note reports experiments suggesting that such problems are not merely theoretical.  It presents empirical evidence that well-known GANs approaches do learn distributions of fairly low support, and thus presumably are not learning the target distribution.  The main technical contribution is a new proposed test, based upon the famous {\em birthday paradox}, for estimating the support size of the generated distribution.

\end{abstract}

\section{Introduction}
From the earliest papers on Generative Adversarial Networks (GANs) the question has been raised whether or not they actually learn the distribution they are trained with (henceforth refered to as the target distribution)?  These methods train a generator deep net
that converts a random seed into a realistic-looking image. Concurrently they train a discriminator deep net to discriminate between its output and real images, which in turn is used to produce gradient feedback to improve the generator net.  In practice the generator deep net starts producing realistic outputs by the end, and the objective approaches its optimal value. But does this mean the deep net has learnt the target distribution of real images? Standard analysis introduced in~\citep{goodfellow2014generative} shows that  given \textquotedblleft sufficiently large\textquotedblright\ generator net, sample size, and computation time the training does succeed in learning the underlying distribution arbitrarily closely (measured in Jensen-Shannon divergence). But this does not settle the question of what happens with realistic sample sizes.  

Note that GANs differ from many previous methods for learning distributions in that they do not provide an estimate of some measure of distributional fit ---e.g., perplexity score. Therefore researchers have probed their performance using surrogate qualitative tests, which were usually designed to rule out the most obvious failure mode of the training, namely,  that the GAN has simply memorized the training data. One test checks the similarity of each generated image to the nearest images in the training set.  Another takes two random seeds $s_1, s_2$ that produced realistic images  and checks the images produced using seeds lying on the line joining $s_1, s_2$. If such \textquotedblleft interpolating\textquotedblright \ images are reasonable and original as well, then this may be taken as evidence that the generated distribution has many novel images. Yet other tests check for existence of semantically meaningful directions in the latent space, meaning that varying the seed along these directions leads to predictable changes e.g., (in case of images of human faces) changes in facial hair, or pose. A recent test proposed by~\cite{wu2016quantitative} checks the log-likelihoods of GANs using Annealed Importance Sampling, whose results indicate the mismatch between generator's distribution and the target distribution.~\cite{poole2016improved} proposed a method to trade-off between sample quality and sample diversity but they don't provide a clear definition or a quantitative metric of sample diversity. 

Recently a new theoretical analysis of GANs with finite sample sizes and finite discriminator size~\cite{arora2017generalization}  revealed the possibility that training may sometimes appear to succeed even if the generator is far from having actually learnt the distribution. Specifically, if the discriminator has size $n$, then the training objective could be $\epsilon$ from optimal even though the output distribution is supported on only $O(n\log n/\epsilon^2)$ images. By contrast one imagines that the target distribution usually must have very large support. For example, the set of all possible images of human faces (a frequent setting in GANs work) must involve all combinations  of hair color/style, facial features, complexion, expression, pose, lighting, race, etc., and thus the possible set of images of faces approaches infinity.   Thus the above paper raises the possibility that the discriminator may be unable to meaningfully distinguish such a diverse target distribution from a trained distribution
with fairly small support. Furthermore, the paper notes that this  failure mode is different from the one usually feared, namely. the generator memorizing training samples. The Arora et al. scenario could involve the trained distribution having small support, and yet all its samples could be completely disjoint from the training samples.


However, the above analysis was only a theoretical one, exhibiting a particular near-equilibrium solution that can happen from certain hyperparameter combinations. It left open the possibility that real-life GANs training avoids such solutions thanks to some not-as-yet-understood
 property of backpropagation or  hyperparameter choices.   Thus further experimental investigation is necessary. And yet it seems difficult at first sight to do such an empirical evaluation of the support size of a distribution: it is not humanly possible to go through hundreds of thousands of images, whereas automated tests of image similarity can be thrown off by small changes in lighting, pose etc.  

The current paper introduces a new empirical test for the support size of the trained distribution, and uses it to find that unfortunately these problems do arise in many well-regarded
GAN training methods.
 
\subsection{Birthday paradox test for support size}

Let's consider a simple test that approximately tests the support size of a {\em discrete} distribution. Suppose a distribution has support $N$. The famous {\em birthday paradox}\footnote{The following is the reason for this name. Suppose there are $k$ people in a room. How large must $k$ be before we have a high likelihood of having two people with the same birthday? Clearly, if we want $100\%$ probability, then $k > 366$ suffices. But assuming people's birthdays are iid draws from some distribution on $[1,366]$ it can be checked that the probability exceeds $50\%$ even when $k$ is as small as $23$.} says that a sample of size about $\sqrt{N}$ would be quite likely  to have a duplicate.  Thus our proposed birthday paradox test for GANs is as follows.

(a) Pick a sample of size $s$ from the generated distribution. (b) Use an automated measure of image similarity to flag the $20$ (say) most similar pairs in the sample. (c) Visually inspect the flagged pairs and check for duplicates. (d) Repeat.

If this test reveals that samples of size $s$ have duplicate images with good probability,  then suspect that the distribution has support size about $s^2$.  

Note that the test is not definitive, because the distribution could assign say a probability $10\%$ to a single image, and be uniform on a huge number of other images. Then the test would be quite likely to find a duplicate even with $20$ samples, even though the true support size is huge. But such nonuniformity (a lot of probability being assigned to a few images) is the only failure mode of the birthday paradox test calculation, and such nonuniformity would itself be considered a failure mode of GANs training. This is captured in the following theorem. 

\begin{theorem} Given a discrete probability distribution $P$ on a set $\Omega$, if there exists a subset $S\subseteq\Omega$ of size $N$ such that $\sum_{s\in S}P(s)\geq\rho$, then the probability of encountering at least one collision among $M$ i.i.d. samples from $P$ is $\geq 1-\exp(-\frac{m^2\rho}{2N})$

\end{theorem}
\begin{proof}
\begin{align}
   &\Pr[\text{there is at least a collision among }M \text{samples}]\notag\\
   \geq& \Pr[\text{there is at least a collision among } M \text{samples}\wedge \text{the collision is within set }S]\notag\\
   &\geq 1 - 1\times(1-\frac{\rho}{N})\times(1-\frac{2\rho}{N})\times\cdots\times(1-\frac{(M-1)\rho}{N})\notag\\
   &\geq 1-\exp(-\frac{M^2\rho}{2N})
\end{align}
The last inequality assumes $M<<N$ and also uses the fact that the worst case is when the $\rho$ probability mass is uniformly distributed on $S$.
\end{proof}

\begin{theorem} Given a discrete probability distribution $P$ on a set $\Omega$, if the probability of encountering at least one collision among $M$ i.i.d. samples from $P$ is $\gamma$, then for $\rho=1-o(1)$, there exists a subset $S\subseteq\Omega$ such that $\sum_{s\in S}P(s)\geq\rho$ with size $\leq\frac{2M\rho^2}{(-3+\sqrt{9+\frac{24}{M}\ln\frac{1}{1-\gamma}})-2M(1-\rho)^2}$

\end{theorem}
\begin{proof}
Suppose $X1, X2, \ldots$ are i.i.d. samples from the discrete distribution $P$. We define $T=\inf\{t\geq2, X_t\in\{X1,X2,
\ldots,X_{t-1}\}\}$ to be the collision time and also we use $\beta=\frac{1}{\Pr[T=2]}=\frac{1}{\sum_{X\in\Omega}P(X)^2}$ as a surrogate for the uniformality of $P$. According Theorem 3 in~\cite{wiener2005bounds}, $Pr[T\geq M]$ can be upper-bouded using $\beta$. Specifically, with $\beta>1000$ and $M\leq 2\sqrt{\beta\ln\beta}$, which is usually true when $P$ is the distribution of a generative model of images,
\begin{align}
   \Pr[T\geq M]\geq\exp(-\frac{M^2}{2\beta}-\frac{M^3}{6\beta^2})
\end{align}

To estimate $\beta$, it follows that
\begin{align}
   \Pr[T\geq M]&=1-\Pr[T\leq M]=1-
   \Pr[\text{there is at least a collision among }M] \notag\\
   &=1-\gamma \geq \exp(-\frac{M^2}{2\beta}-\frac{M^3}{6\beta^2})
\end{align}
which immediately implies 
\begin{align}
   \beta\leq\frac{2M}{-3+\sqrt{9+\frac{24}{M}\ln\frac{1}{1-\gamma}}}=\beta^*
   \label{ineq:beta}
\end{align}
This gives us a upper-bound of the unifomrality of distribution $P$, which we can utilize. Let $S\subseteq\Omega$ be the smallest set with probability mass $\geq\rho$ and suppose it size is $N$. To estimate the largest possible $N$ such that inequality~\ref{ineq:beta} holds, we let
\begin{align}
\frac{1}{(\frac{\rho}{N})^2 N + (1-\rho)^2}\leq\beta^*   
\end{align}

from which we obtain

\begin{align}
N\leq\frac{2M\rho^2}{(-3+\sqrt{9+\frac{24}{M}\ln\frac{1}{1-\gamma}})-2M(1-\rho)^2}
\end{align}

\end{proof}

 \section{Birthday paradox test for GANs: Experimental details}

 In the GAN setting, the distribution is continuous, not discrete. When support size is infinite then  in a finite sample, we should not expect exact duplicate images where every pixel is identical. Thus {\em a priori} one imagines the birthday paradox test to completely not work. But surprisingly, it still works if we look for near-duplicates. Given a finite sample, we select the 20 closest pairs  according to some heuristic metric, thus obtaining a candidate pool of potential near-duplicates inspect. Then we visually identify if any of them would be considered duplicates by humans.
 Our test were done using two datasets, CelebA (faces) and CIFAR-10. 

CelebA~\cite{liu2015faceattributes} is a large-scale face attributes dataset with more than 200K celebrity images (of 11k distinct individuals), each with 40 attribute annotations. The images in this dataset cover large pose variations and background clutter. 
Images of faces seem a good testing ground for our birthday paradox test since we humans are especially attuned to minor differences in faces. 

The CIFAR-10 dataset~\cite{krizhevsky2009learning} consists of 60k 32x32 colour images in 10 classes (airplane, automobile, bird, cat, deer, dog, frog, horse, ship, truck), with 6k images per class. 
 
 For faces,  we found Euclidean distance in pixel space works well
 as a heuristic similarity measure,  probably because the samples are centered and aligned. For CIFAR-10, we pre-train a discriminative Convolutional Neural Net for the full classification problem, and use the top layer representation as an embedding of the image. Heuristic similarity is then measured as the Euclidean distance in the embedding space. Possibly these similarity measures are crude, but note that improving them can only {\em lower} our estimate of the support size of the distribution, since  a better similarity measure can only increase the number of duplicates found. Thus our estimates below should be considered as upper bounds on the support size of the distribution.
 
 We also report some preliminary and inconclusive results on the $3,033,042$ bedrooms images from the LSUN dataset~\cite{yu2015lsun}, and our models are trained on $64\times64$ center crops, which is standard for GANs' training on this dataset.
 
 {\em Note:} Some GANs (and also older methods such as variational autoencoders) implicitly or explicitly apply noise to the training and generated images. This seems useful if the goal is to compute a perplexity score, which involves the model being able to assign a nonzero probability to {\em every} image. Such noised images are usually very blurry and the birthday paradox test does not work well for them, primarily because the automated measure of similarity no longer works well. Even visually judging similarity of noised images is difficult. Thus our experiments work best with GANs that generate sharper, realistic images.

 \subsection{Results  on CelebA dataset}
 We tested the following methods, doing the birthday paradox test  with Euclidean distance in pixel space as the heuristic similarity measure.
 
  \begin{itemize}
 \item DCGAN ---unconditional, with JSD objective as described in~\cite{goodfellow2014generative} and~\cite{radford2015unsupervised}. 
 \item  MIX+ GAN  protocol introduced in~\cite{arora2017generalization}, specifically, MIX+DC-GAN with $3$ mixture components.
 \item Adversarily Learned Inference (ALI)~\cite{dumoulin2017adversarially} (or equivalently BiGANs~\cite{donahue2017adversarial}).\footnote{ALI is probabilistic version of BiGANs, but their architectures are equivalent. So we only tested ALI in our experiments.} 
   \end{itemize}
   
    We find that with probability $\geq50\%$, a batch of about $400$ samples contains at least one pair of duplicates for both DCGAN and MIX+DCGAN. Figure~$\ref{fig:similar_faces}$ give examples duplicates and their nearest neighbors samples (that we could fine) in training set. These results suggest that the support size of the distribution is less than $400^2\approx160000$, which is actually lower than the diversity of the training set, but this distribution is not just memorizing the training set.  
    ALI (or BiGANs) appear to be somewhat more diverse, in that collisions appear with $50\%$ probability only with a batch size of $1000$, implying a support size of a million. This is $5$x the training set, but still much smaller than the diversity one would expect among human faces\footnote{After all most of us know several thousand people, but the only doppelgangers among our acquaintances are twins.}. (For fair comparison, we set the discriminator of ALI (or BiGANs) to be roughly the same in size as that of the DCGAN model, since the results of Section~\ref{subsec:scaling} below suggests that the discriminator size has a strong effect on diversity of the learnt distribution.)    Nevertheless, these tests do support the suggestion in~\cite{dumoulin2017adversarially} and~\cite{donahue2017adversarial} that the bidirectional structure prevents some of the mode collapse observed in usual GANs.  
 
 \subsubsection{Diversity vs Discriminator Size}
 \label{subsec:scaling}
 The analysis of Arora et al~\cite{arora2017generalization} suggested that the support size could be as low as near-linear in the capacity of the discriminator; in other words, there is a near-equilibrium in which a distribution of such a small support could suffice to fool the best discriminator. So it is worth investigating whether training in real life allows generator nets to exploit this \textquotedblleft loophole\textquotedblright\ in the training that we now know is in principle available to them.
 
 While a comprehensive test is beyond the scope of this paper, we did a crude first test with a simplistic version of discriminator size (i.e., capacity). We built DCGANs with increasingly larger discriminators while fixing the other hyper-parameters. The discriminator used here is a 5-layer Convolutional Neural Network such that the number of output channels of each layer is $1\times,2\times,4\times,8\times\textit{dim}$ where $dim$ is  chosen to be $16,32,48,64,80,96,112,128$. Thus the discriminator size  should be proportional to $dim^2$. Fig~\ref{fig:diversity_vs_size} suggests that in this simple setup the diversity of the learnt distribution does indeed grow  near-linearly with the discriminator size. (Note the diversity is seen to plateau, possibly because one needs to change other parameters like depth to meaningfully add more capacity to the discriminator.) 
 \begin{figure}[h]
	\centering
	\includegraphics[width=6.4in]{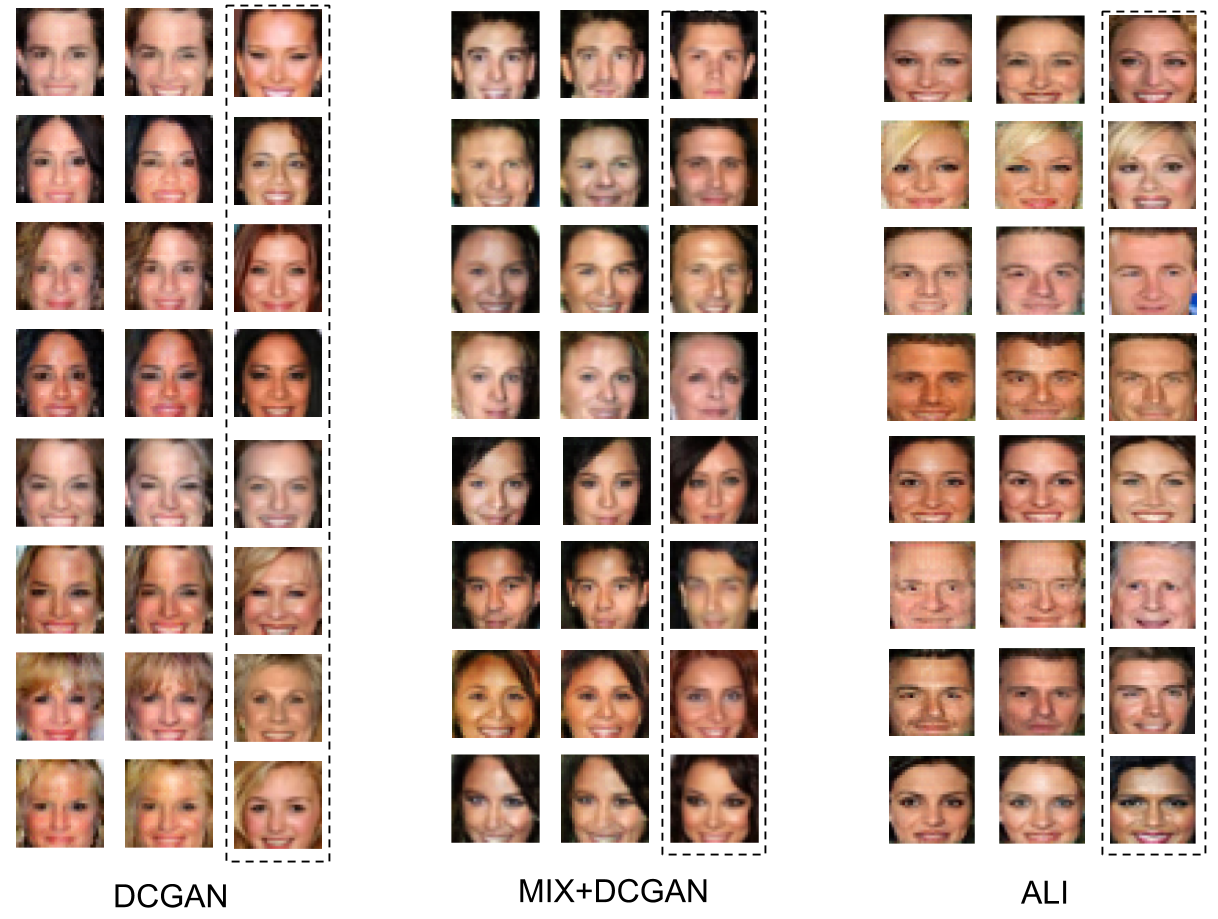}
	\caption{Most similar pairs found in batches of 640 generated faces samples from a DCGAN, a MIX+DCGAN (with 3 component) and an ALI. Each pair is from a different batch. Note that for the first two models, a smaller batch of 400 images is sufficient to detect duplicates with $\geq50\%$ probability. With 640 images, the probability increases to 90\%. Shown in dashed boxes are nearest neighbors in training data.}
	\label{fig:similar_faces}
\end{figure}

 \begin{figure}[h]
	\centering
	\includegraphics[width=6in]{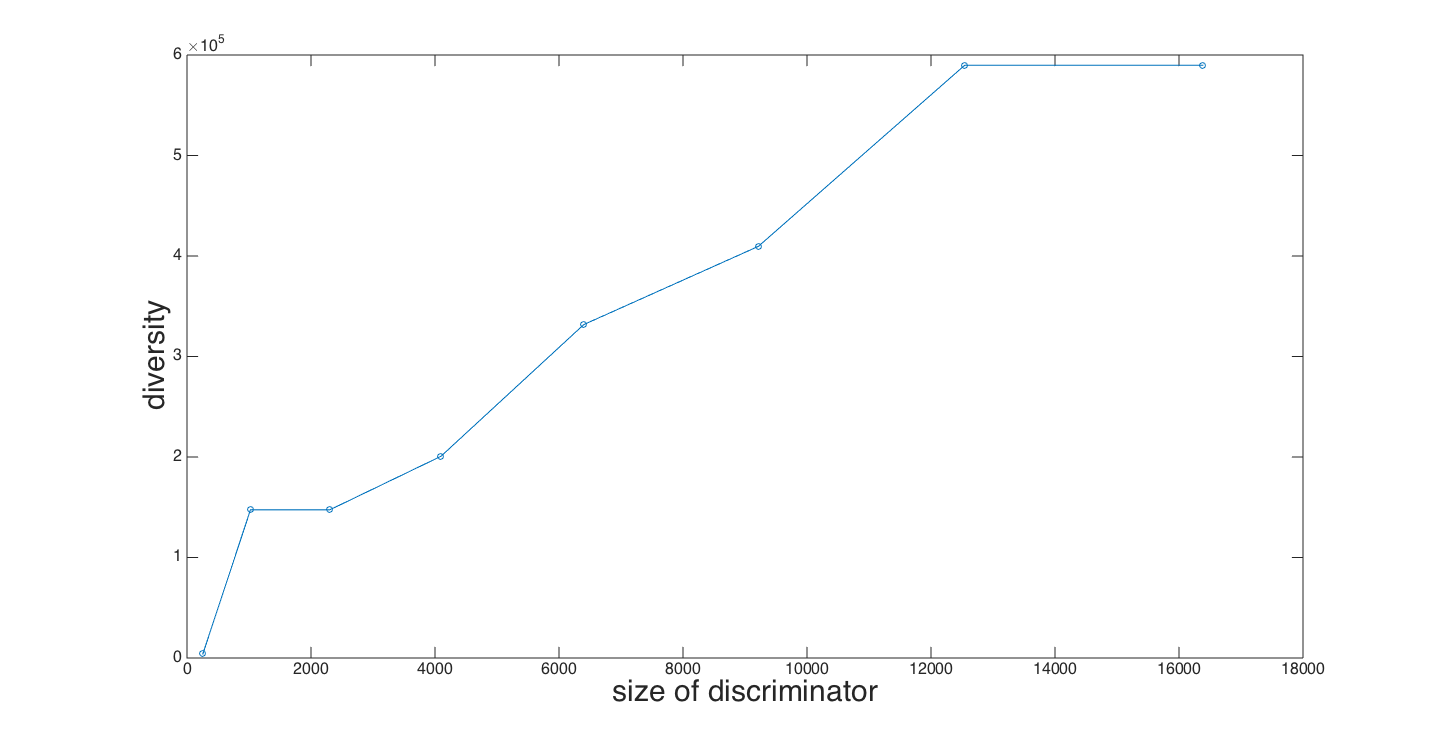}
	\caption{Diversity's dependence on discriminator size. The diversity is measured as the square of 
	the batch size needed to encounter collision w.p. $\geq$ 50\% v.s. size of discriminator. The discriminator size is explained in the main article.}
	\label{fig:diversity_vs_size}
\end{figure}

\subsection{Results for CIFAR-10}

On CIFAR-10, Euclidean distance in pixel space is not informative. So we adopt a classifying CNN with 3 convolutional layers, 2 fully-connected layer and a 10-class soft-max output pretrained with a multi-class classification objective, and use its top layer features as embeddings  for similarity test using Euclidean distance. 
We found, firstly, that the result of the test is affected by the quality of samples. 
If the training uses noised samples (with noise being added either explicitly or implicitly in the objective) then the generated samples are also quite noisy. Then the most similar samples in a batch tend to be blurry blobs of low quality. Indeed, when we test a DCGAN (even the best variant with 7.16 Inception Score~\cite{huang2016stacked}), the pairs returned are mostly blobs. To get meaningful test results, we turn to a Stacked GAN which is the best generative model on CIFAR-10 (Inception Score 8.59 ~\cite{huang2016stacked}). It also generates the most real-looking images. Since this model is trained by conditioning on class label, we measure its diversity within each class separately. The batch sizes needed for duplicates are shown in Table~$\ref{tab:cifar_diversity}$.  Duplicate samples as well as the nearest neighbor to the samples in training set are shown in Figure~$\ref{fig:similar_cifar}$.

\begin{table}[H]

\begin{center}
\begin{small}
\begin{tabular}{c|c|c|c|c|c|c|c|c|c}
\hline
Aeroplane & Auto-Mobile & Bird & Cat & Deer & Dog & Frog & Horse & Ship & Truck\\
\hline
500 & 50 & 500 & 100 & 500 & 300 & 50 & 200 & 500 & 100\\
\hline
\end{tabular}
\caption{Class specific batch size needed to encounter duplicate samples with $>50\%$ probability, from a Stacked GAN trained on CIFAR-10}
\label{tab:cifar_diversity}
\end{small}
\end{center}
\end{table}

 \begin{figure}[h]
	\centering
	\includegraphics[width=6.5in]{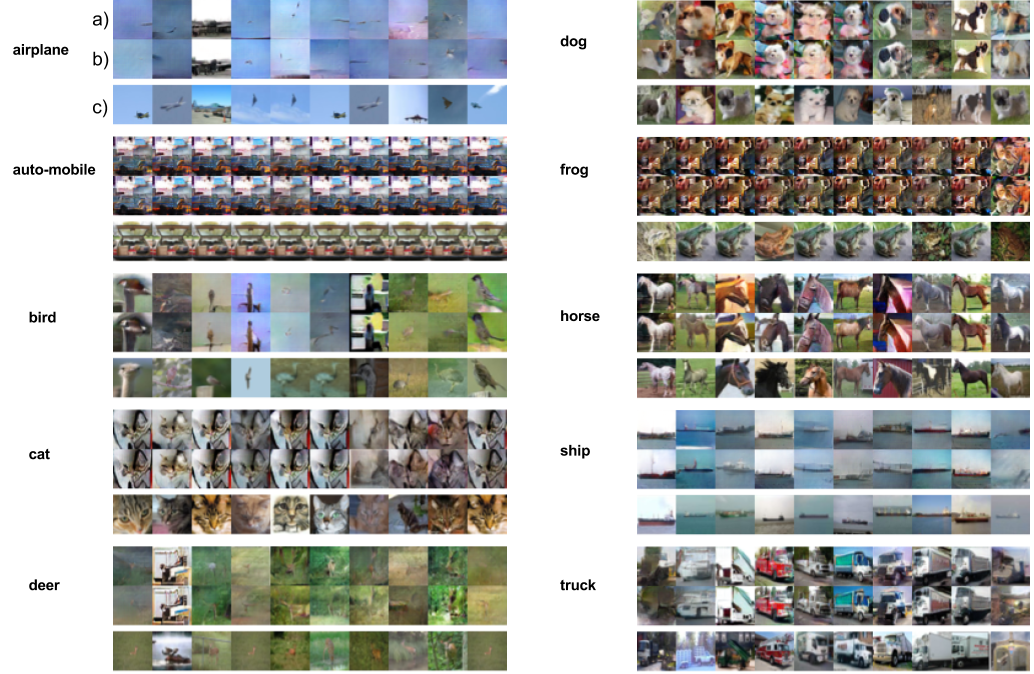}
	\caption{Duplicate pairs found in a batch of 1000 generated CIFAR-10 samples from a Stacked GAN.a)-b):pairs of duplicate samples; c): nearest neighbor of b) in the training set. The batch size 1000 was  not the batch size with 50\% duplicate detection probability. The number 1000 is chosen for the convenience of selecting pairs for visualization.}
	\label{fig:similar_cifar}
\end{figure}

 We also checked whether the duplicate image detected is close to any of the training images. To do so we looked for the nearest neighbor in the training set using our heuristic similarity measure and visually inspected the closest suspects. We find that the closest image is quite different from the duplicate detected, which suggests the issue with GANs is indeed lack of diversity (low support size) instead of memorizing training set.

\subsection{Exploratory results on Bedroom dataset}
We tested DCGANs trained on the Bedroom dataset(LSUN) using Euclidean distance to extract collision candidates since it is impossible to train a CNN classifier on such single-category (bedroom) dataset. We notice that the most similar pairs are likely to be the corrupted samples with the same noise pattern (top-5 collision candidates all contain such patterns). When ignoring the noisy pairs, the most similar "clean" pairs are not even similar according to human eyes. This implies that the distribution puts significant probability on noise patterns, which can be seen as a form of under-fitting  (also reported in the DCGAN paper). We manually counted the number of samples with a fixed noise pattern from a batch of 900 i.i.d samples. We find 43 such corrupted samples among the 900 generated images, which implies $43/900\approx5\%$ probability. 

 \begin{figure}[h]
	\centering
	\includegraphics[width=3.5in]{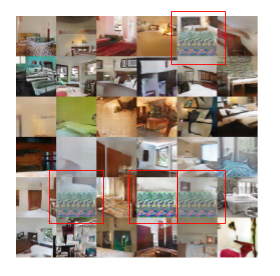}
	\caption{Randomly generated samples from a DCGAN trained on Bedroom dataset. Note that there are corrupted images with a fixed noise pattern (emphasized in red boxes). }
	\label{fig:bedroom_samples}
\end{figure}

\section{Birthday paradox test for VAEs}

 \begin{figure}[h]
	\centering
	\includegraphics[width=5in]{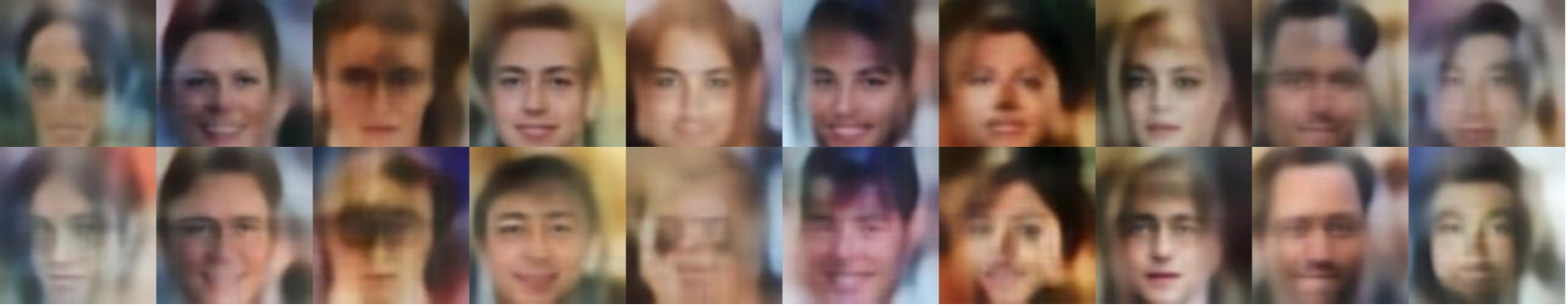}
	\caption{Collision candidates found in Variation Auto-Encoder samples. The duplicated samples are frequently blurry ones because the crucial features (eyes, hair, mouth) of VAE samples are not as distinctive as those of GANs'.}
	\label{fig:vae_collisions}
\end{figure}

Given these findings, it is natural to wonder about the diversity of distributions learned using earlier methods such as Variational Auto-Encoders~\cite{kingma2014auto} (VAEs). Instead of using feedback from the discriminator, these methods train the generator net using feedback from an approximate perplexity calculation. Thus the analysis of~\cite{arora2017generalization} does not apply as is to such methods and it is conceivable they exhibit higher diversity. However, we found the birthday paradox test difficult to run since  samples from a VAE trained on CelebA were not realistic or sharp enough for a human to definitively conclude whether or not two images were almost the same. Fig~\ref{fig:vae_collisions} shows examples of collision candidates found in batches of 400 samples; clearly some indicative parts (hair, eyes, mouth, etc.) are quite blurry in VAE samples.

\section{Conclusions}
We have introduced a new test based upon the Birthday Paradox for testing the diversity of images in a distribution.  Experiments using this test suggest that current GANs approaches, specifically, the ones that produce images of higher visual quality, fall significantly short of learning the target distribution, and in fact the support size of the generated distribution is rather low (mode collapse). The possibility of such a scenario was anticipated in a recent theoretical analysis of~\cite{arora2017generalization}, which showed that the GANs training objective is not capable of preventing mode collapse in the trained distribution. Our rough experiments also suggest ---again in line with the theoretical analysis---that the size of the distribution's support scales near-linearly with discriminator capacity; though this conclusion needs to be rechecked with more extensive experiments.
 

 This combination of theory and empirics raises the open problem of how to change the GANs training objective so that it avoids such mode collapse. Possibly  ALI/BiGANs point to the right direction, since they exhibit somewhat better diversity than the other GANs approaches in our experiments.
 
 Finally, we should consider the possibility that the best use of GANs and related techniques could be feature learning or some other goal, as opposed to distribution learning. This needs further theoretical and empirical exploration.
 
 \section*{Acknowledgements}
 We thank Rong Ge and Cyril Zhang for helpful discussions and Ian Goodfellow for his comments on the manuscript. This research was supported by the National Science Foundation (NSF), Office of Naval Research (ONR), and the Simons Foundation.


\begin{thebibliography}{10}

\bibitem{goodfellow2014generative}
Ian Goodfellow, Jean Pouget-Abadie, Mehdi Mirza, Bing Xu, David Warde-Farley,
  Sherjil Ozair, Aaron Courville, and Yoshua Bengio.
\newblock Generative adversarial nets.
\newblock In {\em Advances in neural information processing systems}, pages
  2672--2680, 2014.

\bibitem{wu2016quantitative}
Yuhuai Wu, Yuri Burda, Ruslan Salakhutdinov, and Roger Grosse.
\newblock On the quantitative analysis of decoder-based generative models.
\newblock {\em arXiv preprint arXiv:1611.04273}, 2016.

\bibitem{poole2016improved}
Ben Poole, Alexander~A Alemi, Jascha Sohl-Dickstein, and Anelia Angelova.
\newblock Improved generator objectives for gans.
\newblock {\em arXiv preprint arXiv:1612.02780}, 2016.

\bibitem{arora2017generalization}
Sanjeev Arora, Rong Ge, Yingyu Liang, Tengyu Ma, and Yi~Zhang.
\newblock Generalization and equilibrium in generative adversarial nets (gans).
\newblock {\em arXiv preprint arXiv:1703.00573}, 2017.

\bibitem{wiener2005bounds}
Michael~J Wiener.
\newblock Bounds on birthday attack times.
\newblock {\em IACR Cryptology ePrint Archive 2005: 318 (2005)}, 2005.

\bibitem{liu2015faceattributes}
Ziwei Liu, Ping Luo, Xiaogang Wang, and Xiaoou Tang.
\newblock Deep learning face attributes in the wild.
\newblock In {\em Proceedings of International Conference on Computer Vision
  (ICCV)}, 2015.

\bibitem{krizhevsky2009learning}
Alex Krizhevsky.
\newblock Learning multiple layers of features from tiny images.
\newblock {\em Citeseer}, 2009.

\bibitem{yu2015lsun}
Fisher Yu, Ari Seff, Yinda Zhang, Shuran Song, Thomas Funkhouser, and Jianxiong
  Xiao.
\newblock Lsun: Construction of a large-scale image dataset using deep learning
  with humans in the loop.
\newblock {\em arXiv preprint arXiv:1506.03365}, 2015.

\bibitem{radford2015unsupervised}
Alec Radford, Luke Metz, and Soumith Chintala.
\newblock Unsupervised representation learning with deep convolutional
  generative adversarial networks.
\newblock {\em arXiv preprint arXiv:1511.06434}, 2015.

\bibitem{dumoulin2017adversarially}
Vincent Dumoulin, Ishmael Belghazi, Ben Poole, Alex Lamb, Martin Arjovsky,
  Olivier Mastropietro, and Aaron Courville.
\newblock Adversarially learned inference.
\newblock In {\em International Conference on Learning Features (ICLR)}, 2017.

\bibitem{donahue2017adversarial}
Jeff Donahue, Philipp Kr{\"a}henb{\"u}hl, and Trevor Darrell.
\newblock Adversarial feature learning.
\newblock In {\em International Conference on Learning Features (ICLR)}, 2017.

\bibitem{huang2016stacked}
Xun Huang, Yixuan Li, Omid Poursaeed, John Hopcroft, and Serge Belongie.
\newblock Stacked generative adversarial networks.
\newblock {\em arXiv preprint arXiv:1612.04357}, 2016.

\bibitem{kingma2014auto}
Diederik~P Kingma and Max Welling.
\newblock Auto-encoding variational bayes.
\newblock In {\em International Conference on Learning Representations (ICLR)},
  2014.

\end{thebibliography}
\end{document}